\documentclass[letterpaper]{article}
\usepackage{uai2019}
\usepackage[margin=1in]{geometry}
\usepackage{url}            
\usepackage{booktabs}       
\usepackage{amsfonts}       
\usepackage{nicefrac}       
\usepackage{microtype}      
\usepackage{bm}
\usepackage{balance}
\usepackage{graphicx}
\usepackage{caption}
\usepackage{enumitem}
\usepackage{subcaption}
\usepackage{epstopdf}

\usepackage{times}
\usepackage{booktabs}
\usepackage{amsthm}
\usepackage{amsmath}
\usepackage{mathrsfs}
\usepackage{amsfonts}
\usepackage{algorithm}
\usepackage{algorithmic}

\theoremstyle{plain}
\newtheorem{thm}{Theorem}
\newtheorem{asu}{Assumption}
\newtheorem{lem}{Lemma}

\newtheorem{cor}[thm]{Corollary}
\newtheorem{defn}{Definition}

\theoremstyle{remark}

\DeclareMathOperator*{\argmin}{argmin}
\def \x {\mathbf{x}}
\def \u {\mathbf{u}}
\def \1 {\left(}
\def \2 {\right)}
\makeatletter

\usepackage{graphicx}
\usepackage{amsmath}
\usepackage{booktabs}
\usepackage{algorithm}
\usepackage{algorithmic}

\usepackage{times}
\usepackage{natbib}
\usepackage{apalike}
\title{Adaptivity and Optimality: A Universal Algorithm for \\ Online Convex Optimization}

\author{Guanghui Wang,\  Shiyin Lu,\  Lijun Zhang\\
National Key Laboratory for Novel Software Technology, Nanjing University, China\\
\texttt{\{wanggh, lusy,zhanglj\}@lamda.nju.edu.cn}} 
\begin{document}
\maketitle
\begin{abstract}
In this paper, we study adaptive online convex optimization, and aim to design a universal algorithm that achieves optimal regret bounds for multiple common types of loss functions. Existing universal methods are limited in the sense that they are optimal for only a subclass of loss functions. To address this limitation, we propose a novel online method, namely Maler, which enjoys the optimal $O(\sqrt{T})$, $O(d\log T)$ and $O(\log T)$ regret bounds for general convex, exponentially concave, and strongly convex functions respectively. The essential idea is to run multiple types of learning algorithms with different learning rates in parallel, and utilize a meta algorithm to track the best one on the fly. Empirical results demonstrate the effectiveness of our method.
\end{abstract}
\section{Introduction}
Online convex optimization (OCO) is a well-established paradigm for modeling sequential decision making \citep{shalev2012online}. The protocol of OCO is as follows: in each round $t$, firstly a learner chooses an action $\mathbf{x}_t$ from a convex set $\mathcal{D}\subseteq\mathbb{R}^d$, at the same time, an adversary reveals a loss function $f_t(\cdot):\mathcal{D}\mapsto\mathbb{R}$, and consequently the learner suffers a loss $f_t(\mathbf{x}_t)$. The goal is to minimize regret, defined as the difference between the cumulative loss of the learner and that of the best action in hindsight \citep{hazan2016introduction}:\vspace{-.09cm}
\begin{equation}
\label{defn:regret}
R(T)=\sum_{t=1}^Tf_t(\mathbf{x}_t)-\min_{\mathbf{x}\in\mathcal{D}}\sum_{t=1}^Tf_t(\mathbf{x}).
\vspace{-.12cm}
\end{equation}
There exist plenty of algorithms for OCO, based on different assumptions about the  loss functions. Without any assumptions beyond convexity and Lipschitz continuity, the classic  Online gradient descent (OGD) with step size on the order of $O(1/\sqrt{t})$ (referred to as convex OGD) guarantees an $O(\sqrt{T})$ regret bound \citep{zinkevich2003online}, where $T$ is the time horizon. While it has been proved minimax optimal for arbitrary convex functions \citep{abernethy2009stochastic}, tighter bounds are still achievable when loss functions are known belong to some easier  categories \emph{in advance}. In particular, for \emph{strongly} convex functions, OGD with step size proportional to $O(1/{t})$ (referred to as strongly convex OGD) achieves an $O(\log T)$ regret bound \citep{hazan2007logarithmic};
for \emph{exponentially concave} functions, the state-of-the-art algorithm is Online Newton step (ONS) \citep{hazan2007logarithmic}, which enjoys an $O(d\log T)$ regret bound, where $d$ is the dimensionality.

This divides OCO into subclasses, relying on manual selections on which algorithm to use for the specific settings. Such requirements, not only are a burden to users, but also hinder the applications to broad domains where the loss functions are on the fly and choosing the right algorithm beforehand is impossible. These issues motivate the innovation of \emph{adaptive} algorithms, which aim to guarantee optimal regret bounds for arbitrary  convex functions, and {automatically} exploit  easier functions whenever possible. The seminal work of  \cite{hazan2008adaptive} propose Adaptive online gradient descent (AOGD), which attains $O(\sqrt{T})$ and $O(\log T)$ regret bounds for convex and strongly convex functions respectively. However, AOGD requires a curvature information of $f_t$ as input in each round, and fails to provide logarithmic regret bound for exponentially concave functions. Another milestone is MetaGrad \citep{van2016metagrad}, which only requires the gradient information, and achieves $O(\sqrt{T\log\log T})$ and $O(d\log T)$ regret bounds for convex and exponentially  concave functions respectively. Although it also implies an $O(d\log T)$ regret for strongly convex functions, there still exists a large $O(d)$ gap from the optimal $O(\log{T})$ regret bound.

Along this line of research, it is therefore natural to ask whether both adaptivity and optimality can be attained simultaneously, or there is an inevitable price in regret to be paid for adaptivity, which was also posed as an open question by \cite{van2016metagrad}. In this paper, we give an affirmative answer by developing a novel online method, namely Maler, which achieves the  optimal regret bounds for all aforementioned three types of loss functions. Inspired by MetaGrad, our method runs multiple expert algorithms in parallel, each of which is configured with a different learning rate,  and employs a meta algorithm to track the best on the fly. However, different from MetaGrad where experts are the same type of OCO algorithms (i.e., a variant of ONS), experts in Maler consists of various types of OCO algorithms (i.e., convex OGD, ONS and strongly convex OGD). Essentially, the goal of MetaGrad is to learn only the optimal learning rate. In contrast, Maler searches for the best OCO algorithm and the optimal learning rate simultaneously. Theoretical analysis shows that, with $O(\log T)$ experts, which is of the same order of that in MetaGrad, Maler achieves $O(\sqrt{T})$, $O(d\log T)$ and $O(\log T)$ regret bounds for convex, exponentially  concave and strongly convex functions respectively. Moreover, we also establish a new type of data-dependent regret bound for Maler, and show that it is no worse than its counterpart of MetaGrad, and in some favorable cases, better. Finally, we conduct experiments on both synthetic and real-world datasets to demonstrate the advantages of our method.

\paragraph{Notation.} Throughout the paper, we use lower case bold face letters to denote vectors, lower case letters to denote scalars, and upper case letters to denote matrices. We use $\|\cdot\|$ to denote the $\ell_2$-norm. For a positive definite matrix $H\in\mathbb{R}^{d\times d}$, the weighted $\ell_2$-norm is denoted by $\|\x\|_H^2=\x^{\top}H\x$. The $H$-weighted projection $\Pi_{\mathcal{D}}^{H}(\x)$ of $\x$ onto $\mathcal{D}$ is defined as $\Pi_{\mathcal{D}}^{H}(\x)=\argmin_{\mathbf{y}\in\mathcal{D}}\|\mathbf{y}-\mathbf{x}\|^2_H$. We denote the gradient of $f_t(\cdot)$ at $\x_t$ as $\mathbf{g}_t$, and the best decision in hindsight as $\x_*=\max\limits_{\x\in\mathcal{D}}\sum_{t=1}^Tf_t(\x)$.

\section{Related Work}
In the literature, there exist various of algorithms for OCO  targeting on a specific type of loss functions. For general convex and strongly convex loss functions, the classic OGD with step size on the order of $O({1}/{\sqrt{t}})$ and $O({1}/{t})$ achieve $O(\sqrt{T})$ and $O(\log T)$ regrets, respectively \citep{zinkevich2003online,hazan2007logarithmic}. For exponentially concave functions, Online Newton step (ONS) attains a regret bound of  $O(d\log T)$ \citep{hazan2007logarithmic}. The above bounds are known to be minimax optimal as matching lower bounds have been established \citep{abernethy2009stochastic,hazan2010optimal}.

To simultaneously deal with multiple types of loss functions, \cite{hazan2008adaptive} propose Adaptive online gradient descent (AOGD), which is later  extended to proximal settings by \cite{do2009proximal}. Both algorithms achieve $O(\sqrt{T})$ and $O(\log{T})$ regret bounds for convex and strongly convex loss functions respectively. Moreover, they have shown superiority over non-adaptive methods in the experiments \citep{do2009proximal}. However, in each round $t$ these algorithms have to be fed with a parameter which depends on the curvature information of $f_t(\cdot)$ at $\x_t$, and cannot achieve the logarithmic regret bound for exponentially concave cases. To address these limitations, \cite{van2016metagrad} propose the Multiple Eta Gradient (MetaGrad), whose basic idea is to run a bunch of ONS algorithms with different learning rates as experts, and then combine them using an expert-tracking algorithm. They show that the regret of MetaGrad for arbitrary convex functions can be simultaneously  bounded by a worst-case bound of $O(\sqrt{T\log\log T})$, and a data-dependant bound of  $O(\sqrt{V_T^{\ell}\ln T})$, where $V^{\ell}_T=\sum_{t=1}^T((\x_*-\x_t)^{\top}\mathbf{g}_t))^2.$ In particular, for strongly convex and exponentially concave functions, the data-dependant bound reduces to $O(d\log T)$.

The above works as well as this paper focus on adapting to different types of loss functions. A related but parallel direction is adapting to structures in \emph{data}, such as low-rank and sparsity. This line of research includes Adagrad \citep{duchi2011adaptive}, RMSprop \citep{tieleman2012lecture}, and Adam \citep{reddi2018convergence}, to name a few. The main idea here is to utilize the gradients observed over time to  dynamically adjust the learning rate or the update direction of gradient descent, and their regret bounds depend on the cumulation of gradients. For general convex functions, the bounds attain $O(\sqrt{T})$ in the worst-case, and become tighter when the gradients are sparse.

Another different direction considers adapting to \emph{changing environments}, where some more stringent criteria are established to measure the performance of algorithms, such as dynamic regret \citep{zinkevich2003online,Dynamic:ICML:13,Adaptive:Dynamic:Regret:NIPS}, which compares the cumulative loss of the learner against any sequence of comparators, and adaptive regret \citep{hazan2007adaptive,daniely2015strongly,jun2016improved,IJCAI:2018:Wang}, which is defined as the maximum regret over any contiguous time interval. In this paper we mainly focus on the minimization of regret, and it an interesting question to explore whether our method can be extended to adaptive and dynamic regret.
\section{Maler}
In this section, we first state assumptions made in this paper, then provide our motivations, and finally present the proposed algorithm as well as its theoretical guarantees.
\subsection{Assumptions and Definitions}
Following previous studies, we introduce some standard assumptions \citep{van2016metagrad} and definitions \citep{boyd2004convex}.
\begin{asu}
\label{asu1}
The gradients of all loss functions $f_1,\dots,f_T$ are bounded by \emph{$G$}, i.e., $\forall t \in \{1,\dots,T\}$, \emph{$\max\limits_{\mathbf{x}\in\mathcal{D}}\|\nabla f_t(\x)\|\leq G$}.
\end{asu}
\begin{asu}
\label{asu2}
The diameter of the decision set is bounded by $D$, i.e., $\max\limits_{\mathbf{x}_1, \mathbf{x}_2\in\mathcal{D}}$ $\|\mathbf{x}_1-\mathbf{x}_2\|\leq D$.
\end{asu}
\begin{defn}
A function $f:\mathcal{D}\mapsto \mathbb{R}$ is convex if  \emph{
\begin{equation}
\begin{split}
\label{defn:convex}
f(\mathbf{x}_1)\geq f(\mathbf{x}_2)+\nabla f(\mathbf{x}_2)^{\top}(\mathbf{x}_1-\mathbf{x}_2), \forall \mathbf{x}_1, \mathbf{x}_2\in\mathcal{D}.
\end{split}
\end{equation}}
\end{defn}
\begin{defn}
\label{defn:stconvex}
A function $f:\mathcal{D}\mapsto \mathbb{R}$ is $\lambda$-strongly convex if  \emph{$\forall \mathbf{x}_1, \mathbf{x}_2\in\mathcal{D}$,
\begin{equation*}
\begin{split}
f(\mathbf{x}_1)\geq f(\mathbf{x}_2)+ \nabla f(\mathbf{x}_2)^{\top}(\mathbf{x}_1-\mathbf{x}_2)+\frac{\lambda}{2}\|\mathbf{x}_1-\mathbf{x}_2\|^2.
\end{split}
\end{equation*}}
\end{defn}
\begin{defn}
\label{def:exp}
A function $f:\mathcal{D}\mapsto \mathbb{R}$ is $\alpha$-exponentially  concave (abbreviated to $\alpha$-exp-concave) if ${\rm exp}(-\alpha f(\mathbf{x}))$ is concave.
\end{defn}

\subsection{Motivation}
\label{Moti}
Our algorithm is inspired by MetaGrad. To help understanding, we first give a brief introduction to the intuition behind this algorithm. Specifically, MetaGrad introduces the following surrogate loss function parameterized by $\eta>0$:
\begin{equation}
\label{elll}
\ell^{\eta}_t(\mathbf{x})=-\eta(\mathbf{x}_t-\mathbf{x})^{\top}\mathbf{g}_t+\eta^2(\mathbf{x}-\mathbf{x}_t)^{\top}\mathbf{g}_t\mathbf{g}_t^{\top}(\mathbf{x}-\mathbf{x}_t).
\end{equation}
The first advantage of the above definition is that $\ell_t^{\eta}$ is exp-concave. Thus, we can apply ONS on $\ell_t^{\eta}$ and obtain the following regret bound with respect to $\ell_t^{\eta}$:
\begin{equation}
\label{eq:adv:1}
\sum_{t=1}^T\ell^{\eta}_t(\mathbf{x}_t)-\min_{\x\in{\mathcal{D}}}\sum_{t=1}^T\ell^{\eta}_t(\mathbf{x})\leq O(d\log T)
\end{equation}
The second advantage is that the regret with respect to the original loss function $f_t$ can be upper bounded in terms of the regret with respect to the defined surrogate loss function $\ell_t^{\eta}$:
\begin{equation}
\label{eq:adv:2}
R(T)\leq \frac{\sum_{t=1}^T\ell^{\eta}_t(\x_t)-\min_{\x\in{\mathcal{D}}}\sum_{t=1}^T\ell^{\eta}_t(\mathbf{x})}{\eta}+\eta V^{\ell}_T
\end{equation}
where $V^{\ell}_T=\sum_{t=1}^T((\x_t-\x_*)^{\top} \mathbf{g}_t)^2.$
Both advantages jointly (i.e., combining (\ref{eq:adv:1}) and (\ref{eq:adv:2})) lead to a regret bound of $O((d \log{T})/\eta + \eta V_T^\ell)$. Therefore, had we known the value of $V_T^\ell$ in advance, we could set $\eta$ as $\Theta(\sqrt{{d\log T}/V_T^{\ell}})$ and obtain the optimal regret bound of $O(\sqrt{{dV_T^{\ell}\log T}})$. However, this is impossible since $V_T^{\ell}$ depends on the whole learning process. To sidestep this obstacle, MetaGrad maintains multiple ONS in parallel each of which targets minimizing the regret with respect to the surrogate loss $\ell_t^{\eta}$ with a different $\eta$, and employs a meta algorithm to track the ONS with the best $\eta$. Theoretical analysis shows that MetaGrad achieves the desired $O(\sqrt{{dV_T^{\ell}\log T}})$ bound.

While the $O(\sqrt{{dV_T^{\ell}\log T}})$ regret bound of MetaGrad can reduce to $O({d\log T})$ for exp-concave functions, it can not recover the $O(\log T)$ regret bound for strongly convex functions. To address this limitation, we propose a new type of surrogate loss function:
\begin{equation}
\label{defn:surrogates}
s^{\eta}_t(\mathbf{x})=-\eta(\mathbf{x}_t-\mathbf{x})^{\top}\mathbf{g}_t+\eta^2G^2\|\mathbf{x}_t-\mathbf{x}\|^2.
\end{equation}
The main advantage of $s^{\eta}_t$ over $\ell^{\eta}_t$ is the strong convexity, which allows us to adopt a strongly convex OGD that takes $s^{\eta}_t$ as the objective loss function and attains
an $O(\log{T})$ regret with respect to $s^{\eta}_t$. On the other hand, the ``upper-bound" property in (\ref{eq:adv:2}) is preserved in the sense that the regret with respect to the original loss $f_t$ can be upper bounded by:
\begin{equation*}
R(T)\leq \frac{\sum_{t=1}^Ts^{\eta}_t(\x_t)-\min_{\x\in{\mathcal{D}}}\sum_{t=1}^Ts^{\eta}_t(\x)}{\eta}+\eta V^{s}_T
\end{equation*}
where $V_T^s=\sum_{t=1}^TG^2\|\x_t-\x_*\|^2$. Thus, the employed strongly convex OGD enjoys a novel data-dependant $O((\log{T})/\eta + \eta V_T^s)$ regret with respect to $f_t$, removing the undesirable factor of $d$. To optimize this bound to $O(\sqrt{V_T^{s}\log T})$, we follow the idea of MetaGrad and run many instances of strongly convex OGD.

Finally, to obtain the optimal $O(\sqrt{T})$ regret bound for general covnex functions, we also introduce a linear surrogate loss function as follows:
\begin{equation}
\label{defn:surrogatec}
c_t(\mathbf{x})=-\eta^c(\mathbf{x}_t-\mathbf{x})^{\top}\mathbf{g}_t+\left(\eta^c GD\right)^2.
\end{equation}
It can be proved that if we run a convex OGD with $c_t(\cdot)$ as the input, its regret with respect to the original loss function $f_t(\cdot)$ can be bounded by $O(1/\eta^c + \eta^c T(GD)^2)$. Here, the optimal $\eta$ is $\Theta(1/(DG\sqrt{T}))$, which depends on only known quantities and thus can be tuned beforehand.

While the idea of incorporating new types of surrogate loss functions to enhance the adaptivity is easy to comprehend, the specific definitions of the two proposed surrogate loss functions in (\ref{defn:surrogates}) and (\ref{defn:surrogatec}) are more involved. In fact, the proposed functions are carefully designed such that besides the aforementioned properties, they also satisfies that
for $\eta\in[0,\frac{2}{3DG}]$,
\begin{equation*}
\exp(-s_t^{\eta}(\x))\leq \exp(-\ell_t^{\eta}(\x))\leq 1+\eta (\x_t-\x)^{\top}\mathbf{g}_t
\end{equation*}
and for $\eta^c=1/(2DG\sqrt{T})$,
\begin{equation*}
\exp(-c_t(\x))\leq 1+\eta^c (\x_t-\x)^{\top}\mathbf{g}_t
\end{equation*}
which are critical to keep the regret caused by the meta algorithm under control and will be made clear in Section \ref{41}.
\begin{algorithm}[t]
\caption{Meta algorithm}
\label{alg:master}
\begin{algorithmic} [1]
\STATE \textbf{Input:} Grid of learning rates $\eta_1,\eta_2,\dots,$ prior weights $\pi_1^c$, $\pi_1^{\eta_1,s}, \pi^{\eta_2,s}_{1},\dots$ and $\pi_1^{\eta_1,\ell}, \pi^{\eta_2,\ell}_{1},\dots$
\FOR{$t=1,\dots,T$}
\STATE Get predictions $\mathbf{x}_t^c$ from Algorithm \ref{alg:slave:c}, and $\mathbf{x}^{\eta,\ell}_t$, $\mathbf{x}^{\eta,s}_t$ from Algorithms \ref{alg:slave} and \ref{alg:slave:sc} for all  $\eta$
\STATE Play $\mathbf{x}_t=\frac{\pi_t^c\eta^c\mathbf{x}_t^c+\sum_{\eta}(\pi_t^{\eta,s}\eta\mathbf{x}_t^{\eta,s}+\pi_t^{\eta,\ell}\eta\mathbf{x}_t^{\eta,\ell})}{\pi_t^c\eta^c+\sum_{\eta}(\pi_t^{\eta,s}\eta+\pi_t^{\eta,\ell}\eta)}$
\STATE Observe gradient $\mathbf{g}_t$ and send it to all experts
\STATE Update weights:\\
$\pi_{t+1}^c=\frac{\pi_t^{c}e^{-c_t\left(\mathbf{x}_t^{c}\right)}}{\Phi_t}$\\
$\pi_{t+1}^{\eta,s}=\frac{\pi_t^{\eta,s}e^{-s_t^{\eta}\left(\mathbf{x}_t^{\eta,s}\right)}}{\Phi_t}$ for all  $\eta$\\
$\pi_{t+1}^{\eta,\ell}=\frac{\pi_t^{\eta,\ell}e^{-\ell_t^{\eta}\left(\mathbf{x}_t^{\eta,\ell}\right)}}{\Phi_t} $ for all $\eta$\\
where
\begin{equation*}
\begin{split}
\Phi_t=&\sum_{\eta}\left(\pi_t^{\eta,s}e^{-s_t^{\eta}\left(\mathbf{x}_t^{\eta,s}\right)}+\pi_t^{\eta,\ell}e^{-\ell_t^{\eta}\left(\mathbf{x}_t^{\eta,\ell}\right)}\right)\\
&+\pi_t^ce^{-c_t\left(\mathbf{x}_t^c\right)}
\end{split}
\end{equation*}
\ENDFOR
\end{algorithmic}
\end{algorithm}
\begin{algorithm}[t]
\caption{Convex expert algorithm}
\label{alg:slave:c}
\begin{algorithmic} [1]
\STATE $\mathbf{x}^{c}_1=\mathbf{0}$
\FOR{$t=1,\dots,T$}
\STATE Send $\mathbf{x}_t^{c}$ to Algorithm \ref{alg:master}
\STATE Receive gradient $\mathbf{g}_t$ from Algorithm \ref{alg:master}
\STATE Update $\mathbf{x}^{c}_{t+1}=\Pi^{I_d}_{\mathcal{D}}\left(\mathbf{x}_t^{c}-\frac{D}{\eta^cG\sqrt{t}}{\nabla c_t(\x_t^c)}\right),$ where $\nabla c_t(\x_t^c)=\eta^c\mathbf{g}_t$
\ENDFOR
\end{algorithmic}
\end{algorithm}
\begin{algorithm}[h]
\caption{Exp-concave expert algorithm}
\label{alg:slave}
\begin{algorithmic} [1]
\STATE \textbf{Input:} Learning rate $\eta$
\STATE $\mathbf{x}^{\eta,\ell}_1=\mathbf{0},$ $\beta=\frac{1}{2}\min\left\{\frac{1}{4G^{\ell}D},1\right\}$, where $G^{\ell}=\frac{7}{25D}$, $\Sigma_1=\frac{1}{\beta^2D^2} I_d$
\FOR{$t=1,\dots,T$}
\STATE Send $\mathbf{x}_t^{\eta,\ell}$ to  Algorithm \ref{alg:master}
\STATE Receive gradient $\mathbf{g}_t$ from Algorithm \ref{alg:master}
\STATE Update 
\begin{equation*}
\begin{split}
\Sigma_{t+1}=&\Sigma_t+\nabla \ell^{\eta}_t\left(\x_t^{\eta,\ell}\right)\left(\nabla \ell^{\eta}_t\left(\x_t^{\eta,\ell}\right)\right)^{\top}\\
\mathbf{x}^{\eta,\ell}_{t+1}=&\Pi_{\mathcal{D}}^{\Sigma_{t+1}}\left(\x^{\eta,\ell}_t-\frac{1}{\beta}\Sigma_{t+1}^{-1}\nabla \ell^{\eta}_t\left(\x_t^{\eta,\ell}\right)\right)
\end{split}
\end{equation*}
where $\nabla \ell^{\eta}_t(\x_t^{\eta,\ell})=\eta \mathbf{g}_t+2\eta^2\mathbf{g}_t\mathbf{g}_t^{\top}(\x_t^{\eta,\ell}-\x_t)$
\ENDFOR
\end{algorithmic}
\end{algorithm}
\begin{algorithm}[t]
\caption{Strongly convex expert algorithm}
\label{alg:slave:sc}
\begin{algorithmic} [1]
\STATE \textbf{Input:} Learning rate $\eta$
\STATE $\mathbf{x}^{\eta,s}_1=\mathbf{0}$
\FOR{$t=1,\dots,T$}
\STATE Send $\mathbf{x}_t^{\eta,s}$ to  Algorithm \ref{alg:master}
\STATE Receive gradient $\mathbf{g}_t$ from Algorithm \ref{alg:master}
\STATE Update $$\mathbf{x}^{\eta,s}_{t+1}=\Pi^{I_d}_{\mathcal{D}}\left(\mathbf{x}_t^{\eta,s}-\frac{1}{2\eta^2 G^2t}\nabla s_t^{\eta}\left(\x_t^{\eta,s}\right)\right)$$ where $\nabla s_t^{\eta}(\x_t^{\eta,s})=\eta \mathbf{g}_t+2\eta^2G^2\left(\mathbf{x}_t^{\eta,s}-\mathbf{x}_t\right)$
\ENDFOR
\end{algorithmic}
\end{algorithm}
\subsection{The algorithm}
Our method, named {\bf M}ultiple sub-{\bf a}lgorithms  and {\bf le}arning {\bf r}ates (Maler), is a two-level hierarchical structure: at the lower level, a set of experts run in parallel, each of which is configured with a different learning algorithm (Algorithm \ref{alg:slave:c}, \ref{alg:slave}, or \ref{alg:slave:sc}) and learning rate. At the higher level, a meta algorithm (Algorithm \ref{alg:master}) is employed to track the best expert based on empirical performances of the experts.

{\bf Meta Algorithm.}  Tracking the best expert is a well-studied problem, and our meta algorithm is built upon the titled exponentially weighted average \citep{van2016metagrad}. The inputs of the meta algorithm are learning rates and  prior weights of the experts. In each round $t$, the meta algorithm firstly receives actions from all experts (Step 3), and then combines these actions by using exponentially weighted average (Step 4). The weights of the experts are titled by their own $\eta$, so that those experts with larger learning rates will be assigned with larger weights. After observing the gradient at $\x_t$ (Step 5), the meta algorithm updates the weight of each expert via an exponential weighting scheme (Step 6).

{\bf Experts.}   Experts are themselves \emph{non-adaptive} algorithms, such as OGD and ONS. In each round $t$, each expert sends its action to the meta algorithm, then receives a gradient vector from the meta algorithm, and finally updates the action based on the received vector. To optimally handle general convex, exp-concave, and strongly convex functions simultaneously, we design three types of experts as follows:
\begin{itemize}

\setlength{\parsep}{0pt}
\setlength{\parskip}{0pt}
\item Convex expert. As discussed in Section \ref{Moti}, there is no need to search for the optimal learning rate in convex cases and thus we only run one convex OGD (Algorithm \ref{alg:slave:c}) on the convex surrogate loss function $c_t(\mathbf{x})$ in (\ref{defn:surrogatec}).  We denote its action in round $t$ as $\x_t^{c}$. Its prior weight $\pi_1^{c}$ and learning rate $\eta^c$ are set to be ${1}/{3}$ and $1/(2GD\sqrt{T})$, respectively.
\item Exp-concave experts. We keep  $\left\lceil \frac{1}{2}\log T\right\rceil+1$ exp-concave experts, each of which is a standard ONS (Algorithm \ref{alg:slave}) running on an exp-concave surrogate loss function $\ell_t^{\eta}(\cdot)$ in \eqref{elll} with a different $\eta$. We denote its output in round $t$ as $\x_{t}^{\eta,\ell}$. For expert $i=0,1,2,...,\left\lceil \frac{1}{2}\log T\right\rceil,$ its learning rate and prior weight are assigned as follows:
\begin{equation*}
\eta_i=\frac{2^{-i}}{5DG},\ {\rm and}\ \pi^{\eta_i,\ell}_{1}=\frac{C}{3(i+1)(i+2)},
\end{equation*}
where $C=1+1/\left(1+\left\lceil\frac{1}{2}\log T\right\rceil\right)$ is a normalization parameter.
\item  Strongly convex experts. We maintain $\left\lceil \frac{1}{2}\log T\right\rceil+1$ strongly convex experts. In each round $t$, every expert takes an strongly convex surrogate loss $s_t^{\eta}(\cdot)$ in \eqref{defn:surrogates} (with different $\eta$) as the loss function, and adopts strongly convex OGD (Algorithm \ref{alg:slave:sc}) to update its action, denoted as $\x_t^{\eta,s}$. For $ i=0,1,2,...,\left\lceil \frac{1}{2}\log T\right\rceil$, we configure the $i$-th strongly expert as follows:
\begin{equation*}
\eta_i=\frac{2^{-i}}{5DG},\ {\rm and}\ \pi^{\eta_i,s}_{1}=\frac{C}{3(i+1)(i+2)}.
\end{equation*}
\end{itemize}
\paragraph{Computational Complexity.} The computational complexity of Maler is dominated by its experts. If we ignore the projection procedure, the run time of Algorithms  \ref{alg:slave:c}, \ref{alg:slave} and \ref{alg:slave:sc} are $O(d)$, $O(d^2)$ and $O(d)$ per iteration respectively. Combining with the number of experts, the total run time of Maler is $O(d^2\log T)$, which is of the same order of that in MetaGrad. When taking the projection into account, we note that it can be computed efficiently for many convex bodies  used  in practical applications such as $d$-dimensional balls, cubes and simplexes \citep{hazan2007logarithmic}. To put it more concrete, when the convex body is a $d$-dimensional ball, projections in Algorithms  \ref{alg:slave:c}, \ref{alg:slave}, and \ref{alg:slave:sc} require $O(d)$, $O(d^3)$, and $O(d)$ time respectively \citep{van2016metagrad}, and consequently the total computational complexity of Maler is $O(d^3 \log T)$, which is also the same as that of MetaGrad.
\subsection{Theoretical Guarantees }
\begin{thm}
\label{th1}
Suppose assumptions \ref{asu1} and \ref{asu2} hold. Let $V_T^s=G^2\sum_{t=1}^T\|\x_t-\x_*\|^2$, and  $V_T^{\ell}=\sum_{t=1}^T((\x_t-\x_*)^{\top}\mathbf{g}_t)^2$. Then the regret of Maler is simultaneously bounded by
\begin{equation}
\label{t1c}
R(T)\leq 2(1+\ln 3)GD\sqrt{T}=O\left(\sqrt{T}\right),
\end{equation}

\begin{equation}
\begin{split}
\label{t1ec}
&R(T)\\
\leq& 3\sqrt{V_T^{\ell}\left(2\ln\left(\sqrt{3}\left(\frac{1}{2}\log_2 T+3\right)\right)+10d\log T\right)}\\ &+10GD\left(2\ln\left(\sqrt{3}\left(\frac{1}{2}\log_2 T+3\right)\right)+10d\log T\right)\\
=&O\left(\sqrt{V_T^{\ell}d\log T}\right),
\end{split}
\end{equation}
and
\begin{equation}
\begin{split}
\label{t1sc}
&R(T)\\
\leq& 3\sqrt{V_T^s\left(2\ln\left(\sqrt{3}\left(\frac{1}{2}\log_2 T+3\right)\right)+1+\log T\right)}\\ &+10GD\left(2\ln\left(\sqrt{3}\left(\frac{1}{2}\log_2 T+3\right)\right)+1+\log T\right)\\
=&O\left(\sqrt{V_T^s\log T}\right).
\end{split}
\end{equation}
\end{thm}
\paragraph{Remark.} Theorem \ref{th1} implies that, similar to MetaGrad, Maler can be upper bounded by $O(\sqrt{V_{T}^{\ell}d\log T})$. Hence, the conclusions of MetaGrad under some fast rates examples such as Bernstein condition still hold for Maler. Moreover, Theorem \ref{th1} shows that Maler also enjoys a new type of data-dependant bound $O(\sqrt{V_{T}^{s}\log T})$, and thus may perform better than MetaGrad in some high dimensional cases such that $V_T^{s}\ll dV_T^{\ell}$.

Next, based on Theorem 1, we derive the following regret bounds for strongly convex and exp-concave loss functions, respectively.
\begin{cor}
\label{cor}
For $\lambda$-strongly convex functions, Theorem \ref{th1}  implies
\begin{equation*}
\begin{split}
R(T)\leq & \left(10GD+\frac{9G^2}{2\lambda}\right)\bigg(2\ln\left(\sqrt{3}\left(\frac{1}{2}\log_2 T+3\right)\right)\\
&+1+\log T\bigg)=O\left(\frac{1}{\lambda}\log T\right).
\end{split}
\end{equation*}
For $\alpha$-exp-concave functions, let $\beta=\frac{1}{2}\min\left\{{\alpha},\frac{1}{4GD}\right\}$, and Theorem \ref{th1}  implies
\begin{equation*}
\begin{split}
R(T)\leq& \left(10GD+\frac{9}{2\beta}\right)\bigg(2\ln\left(\sqrt{3}\left(\frac{1}{2}\log_2 T+3\right)\right)\\
&+10d\log T\bigg)=O\left(\frac{1}{\alpha}d\log T\right).
\end{split}
\end{equation*}
\end{cor}
\paragraph{Remark.} Theorem \ref{th1} and Corollary \ref{cor} indicate that our proposed algorithm achieves the worst-case optimal $O(\sqrt{T})$, $O(d\log T)$ and $O(\log T)$ regret bounds for convex, exponentially concave and strongly convex functions respectively. In contrast, the regret bounds of MetaGrad for the three types of loss functions are $O(\sqrt{T\log\log T})$, $O(d\log T)$ and $O(d\log T)$, which are suboptimal for convex and strongly convex functions.
\section{Regret Analysis}
The regret of Maler can be generally decomposed into two components, i.e., the regret of  the meta algorithm (meta regret) and the regrets of expert algorithms (expert regret). We firstly upper bound the two parts separately, and then analyse their composition to prove Theorem 1.
\subsection{Meta Regret}
\label{41}
We define meta regret as the  difference between the cumulative surrogate losses of the actions of the meta algorithm (i.e., $\x_t$s) and that of the actions from a specific expert, which measures the learning ability of the Meta algorithm. For meta regret, we introduce the following lemma.
\begin{lem}
\label{lemma:master}
For every grid point $\eta$, we have
\begin{equation}
\label{lemma1e1}
\sum_{t=1}^Ts_t^{\eta}(\mathbf{x}_t)-\sum_{t=1}^Ts_t^{\eta}(\mathbf{x}^{\eta,s}_t)\leq2\ln\left(\sqrt{3}\left(\frac{1}{2}\log_2 T+3\right)\right)
\end{equation}
\begin{equation}
\label{lemma1e2}
\sum_{t=1}^T\ell_t^{\eta}(\mathbf{x}_t)-\sum_{t=1}^T\ell_t^{\eta}(\mathbf{x}^{\eta,\ell}_t)\leq2\ln\left(\sqrt{3}\left(\frac{1}{2}\log_2 T+3\right)\right)
\end{equation}
and
\begin{equation}
\label{lemma1e3}
\sum_{t=1}^Tc_t(\mathbf{x}_t)-\sum_{t=1}^Tc_t(\mathbf{x}^{c}_t)\leq\ln3.
\end{equation}
\end{lem}
\begin{proof}
We firstly introduce three inequalities. For every grid point $\eta$,
\begin{equation}
\begin{split}
\label{ineq:meta_1}
e^{-s_t^{\eta}(\mathbf{x}_t^{\eta,s})}\overset{\eqref{defn:surrogates}}{=}&e^{\eta(\mathbf{x}_t-\mathbf{x}_t^{\eta,s})^{\top}\mathbf{g}_t-\eta^2G^2\|\mathbf{x}_t-\mathbf{x}_t^{\eta,s}\|^2}\\
\leq& e^{\eta(\mathbf{x}_t-\mathbf{x}_t^{\eta,s})^{\top}\mathbf{g}_t-\left(\eta (\mathbf{x}_t-\mathbf{x}_t^{\eta,s})^{\top}\mathbf{g}_t\right)^2}\\
\leq& 1+\eta (\mathbf{x}_t-\mathbf{x}_t^{\eta,s})^{\top}\mathbf{g}_t
\end{split}
\end{equation}
where the first inequality follows from Cauchy-Schwarz inequality, and the second inequality is due to $e^{x-x^2}\leq1+x$ for any $x\geq-\frac{2}{3}$ \citep{van2016metagrad}. Applying similar arguments, we have  for every grid point $\eta$,
\begin{equation}
\begin{split}
\label{l1e2}
e^{-\ell_t^{\eta}(\mathbf{x}_t^{\eta,\ell})}\overset{\eqref{elll}}{=}&e^{\eta(\mathbf{x}_t-\mathbf{x}_t^{\eta,\ell})^{\top}\mathbf{g}_t-\left(\eta (\mathbf{x}_t-\mathbf{x}_t^{\eta,\ell})^{\top}\mathbf{g}_t\right)^2}\\
\leq& 1+\eta (\mathbf{x}_t-\mathbf{x}_t^{\eta,\ell})^{\top}\mathbf{g}_t
\end{split}
\end{equation}
and
\begin{equation}
\begin{split}
\label{l1e3}
e^{-c_t(\mathbf{x}_t^{c})}\overset{\eqref{defn:surrogatec}}{=}&e^{\eta^c(\mathbf{x}_t-\mathbf{x}_t^c)^{\top}\mathbf{g}_t-\left(\eta^cGD\right)^2}\\
\leq& e^{\eta^c(\mathbf{x}_t-\mathbf{x}_t^{c})^{\top}\mathbf{g}_t-\left(\eta^c (\mathbf{x}_t-\mathbf{x}_t^{c})^{\top}\mathbf{g}_t\right)^2}\\
\leq& 1+\eta^c (\mathbf{x}_t-\mathbf{x}_t^{c})^{\top}\mathbf{g}_t.
\end{split}
\end{equation}
Note that by definition of $\eta^c$ we have $\eta^c(\x_t-\x_t^c)^{\top}\mathbf{g}_t>-\frac{1}{2}$.\\
Now we are ready to prove Lemma \ref{lemma:master}. Define potential function
\begin{equation}
\begin{split}
\label{phi}
\Phi_T=&\sum_{\eta}\left(\pi_1^{\eta,s}e^{-\sum_{t=1}^Ts_t^{\eta}(\mathbf{x}_t^{\eta,s})}+\pi_1^{\eta,\ell}e^{-\sum_{t=1}^T\ell_t^{\eta}(\mathbf{x}_t^{\eta,\ell})}\right)\\
&+\pi^c_1e^{-\sum_{t=1}^T{c_t(\mathbf{x}_t^{c})}}.
\end{split}
\end{equation}
We have
\begin{equation}
\begin{split}
\label{ineq:meta_22}
&\Phi_{T+1}-\Phi_{T}\\
=&\sum_{\eta}\pi_1^{\eta,s}e^{-\sum_{t=1}^Ts_t^{\eta}(\mathbf{x}_t^{\eta,s})}\left(e^{-s_{T+1}^{\eta}(\mathbf{x}_{T+1}^{\eta,s})}-1\right)\\
&+\sum_{\eta}\pi_1^{\eta,\ell}e^{-\sum_{t=1}^T\ell_t^{\eta}(\mathbf{x}_t^{\eta,\ell})}\left(e^{-\ell_{T+1}^{\eta}(\mathbf{x}_{T+1}^{\eta,\ell})}-1\right)\\
&+\pi_1^ce^{-\sum_{t=1}^Tc_t(\mathbf{x}_t^c)}\left(e^{-c_t(\mathbf{x}^c_{T+1})}-1\right)\\
\leq&\sum_{\eta}\pi_1^{\eta,s}e^{-\sum_{t=1}^Ts_t^{\eta}(\mathbf{x}_t^{\eta,s})}\eta (\mathbf{x}_{T+1}-\mathbf{x}_{T+1}^{\eta,s})^{\top}\mathbf{g}_t\\
&+\sum_{\eta}\pi_1^{\eta,\ell}e^{-\sum_{t=1}^T\ell_t^{\eta}(\mathbf{x}_{t}^{\eta,\ell})}\eta (\mathbf{x}_{T+1}-\mathbf{x}_{T+1}^{\eta,\ell})^{\top}\mathbf{g}_t\\
&+\pi_1^ce^{-\sum_{t=1}^Tc_t(\mathbf{x}_t^c)}\eta^c (\mathbf{x}_{T+1}-\mathbf{x}_{T+1}^{c})^{\top}\mathbf{g}_t\\
=&\left({a}_T\mathbf{x}_{T+1}-\mathbf{b}_T\right)^{\top}\mathbf{g}_t
\end{split}
\end{equation}
where the inequality is due to \eqref{ineq:meta_1}, \eqref{l1e2}, and \eqref{l1e3}, and
\begin{equation*}
\begin{split}
{a}_T=&\sum_{\eta}\pi_1^{\eta,s}e^{-\sum_{t=1}^Ts_t^{\eta}(\mathbf{x}_t^{\eta,s})}\eta+\pi_1^ce^{-\sum_{t=1}^Tc_t(\mathbf{x}_t^c)}\eta^c\\
&+\sum_{\eta}\pi_1^{\eta,\ell}e^{-\sum_{t=1}^T\ell_t^{\eta}(\mathbf{x}_t^{\eta,\ell})}\eta\\
\mathbf{b}_T=&\sum_{\eta}\pi_1^{\eta,\ell}e^{-\sum_{t=1}^T\ell_t^{\eta}(\mathbf{x}_t^{\eta,\ell})}\eta\mathbf{x}_{T+1}^{\eta,\ell}\\
&+\pi_1^ce^{-\sum_{t=1}^Tc_t(\mathbf{x}_t^c)}\eta^c\mathbf{x}_{T+1}^c\\
&+\sum_{\eta}\pi_1^{\eta,s}e^{-\sum_{t=1}^Ts_t^{\eta}(\mathbf{x}_t^{\eta,s})}\eta\mathbf{x}_{T+1}^{\eta,s}
\end{split}
\end{equation*}
On the other hand, by the update rule of $\mathbf{x}_t$, we have
\begin{equation}
\begin{split}
\label{ineq:meta_33}
\mathbf{x}_{T+1}=&\frac{\sum_{\eta}(\pi_{T+1}^{\eta,s}\eta\mathbf{x}_{T+1}^{\eta,s}+\pi_{T+1}^{\eta,\ell}\eta\mathbf{x}_{T+1}^{\eta,\ell})+\pi_{T+1}^c\eta^c\mathbf{x}_{T+1}^c}{\sum_{\eta}(\pi_{T+1}^{\eta,s}\eta+\pi_{T+1}^{\eta,\ell}\eta)+\pi_{T+1}^c\eta^c}\\
=&\frac{\mathbf{b}_T}{a_T}
\end{split}
\end{equation}
where the second equality comes from Step 6 of Algorithm \ref{alg:master}, and note that $\pi_{t+1}^{c}, \pi_{t+1}^{\eta,\ell}$ and $\pi_{t+1}^{\eta,s}$ share the same denominator.
Plugging \eqref{ineq:meta_33} into \eqref{ineq:meta_22}, we get
$$\Phi_{T+1}-\Phi_{T}\leq0$$
which implies that
\begin{equation}
\label{phi2}
1=\Phi_0\geq\Phi_1\geq\dots\geq\Phi_T.
\end{equation}
Note that all terms in the the definition of $\Phi_T$ \eqref{phi} are positive. Combining with \eqref{phi2}, it indicates that these terms are less than 1. Thus,
\begin{equation*}
0\leq -\ln\left(\pi_1^{\eta,s}e^{-\sum_{t=1}^Ts_t^{\eta}(\mathbf{x}_t^{\eta,s})}\right)=\sum_{t=1}^Ts_t^{\eta}(\mathbf{x}_t^{\eta,s})+\ln\frac{1}{\pi_1^{\eta,s}}
\end{equation*}
\begin{equation*}
0\leq -\ln\left(\pi_1^{\eta,\ell}e^{-\sum_{t=1}^T\ell_t^{\eta}(\mathbf{x}_t^{\eta,\ell})}\right)=\sum_{t=1}^T\ell_t^{\eta}(\mathbf{x}_t^{\eta,\ell})+\ln\frac{1}{\pi_1^{\eta,\ell}}
\end{equation*}
and
\begin{equation*}
0\leq -\ln\left(\pi^c_1e^{-\sum_{t=1}^T{c_t(\mathbf{x}_t^{c})}}\right)= \sum_{t=1}^Tc_t(\mathbf{x}_t^{c})+\ln\frac{1}{\pi_1^{c}}.
\end{equation*}
We finish the proof by noticing that for every  grid point $\eta$,
\begin{equation*}
\begin{split}
\ln \frac{1}{\pi_1^{\eta,s}}\leq& \ln\left(3\left(\left\lceil \frac{1}{2}\log T\right\rceil+1\right)\left(\left\lceil \frac{1}{2}\log T\right\rceil+2\right)\right)\\
\leq& 2\ln\left(\sqrt{3}\left(\frac{1}{2}\log_2 T+3\right)\right)
\end{split}
\end{equation*}
\begin{equation*}
\begin{split}
\ln \frac{1}{\pi_1^{\eta,\ell}}\leq& \ln\left(3\left(\left\lceil \frac{1}{2}\log T\right\rceil+1\right)\left(\left\lceil \frac{1}{2}\log T\right\rceil+2\right)\right)\\
\leq& 2\ln\left(\sqrt{3}\left(\frac{1}{2}\log_2 T+3\right)\right)
\end{split}
\end{equation*}
and $\ln \frac{1}{\pi_1^{c}}=\ln 3$.
\end{proof}
\subsection{Expert Regret}
For the regret of each expert, we have the following lemma. The proof is postponed to the appendix.
\begin{lem}
\label{lem2}
For every grid point $\eta$ and any $\u\in\mathcal{D}$, we have \emph{
\begin{equation}
\label{lemma2e1}
\sum_{t=1}^Ts_t^{\eta}(\mathbf{x}_t^{\eta,s})-\sum_{t=1}^Ts_t^{\eta}(\u)\leq 1+\log T
\end{equation}}
\emph{
\begin{equation}
\label{lemma2e2}
\sum_{t=1}^T\ell_t^{\eta}(\mathbf{x}_t^{\eta,\ell})-\sum_{t=1}^T\ell_t^{\eta}(\u)\leq 10d\log T
\end{equation}
and
\begin{equation}
\label{lemma2e3}
\sum_{t=1}^Tc_t(\mathbf{x}_t^c)-\sum_{t=1}^Tc_t(\u)\leq\frac{3}{4}.
\end{equation}}
\end{lem}
\subsection{Proof of Theorem 1}
In the following, we combine the regret analysis of the meta and expert algorithms to prove Theorem 1.
\begin{proof}
To get the $\sqrt{T}$ bound of \eqref{t1c},  we upper bound the regret by using the properties of $c_t$ as follows.
\begin{equation*}
\begin{split}
&R(T)\\
\overset{\eqref{defn:regret}}{=}&\sum_{t=1}^Tf_t(\mathbf{x}_t)-\sum_{t=1}^Tf_t(\mathbf{x}_*)\\
\overset{\eqref{defn:convex}}{\leq}& \sum_{t=1}^T\mathbf{g}_t^{\top}(\mathbf{x}_t-\mathbf{x}_*)\\
\overset{\eqref{defn:surrogatec}}{=}&\frac{\sum_{t=1}^T-c_t(\textbf{x}_*)+\sum_{t=1}^T(\eta^cGD)^2}{\eta^c}\\
=&\frac{\sum_{t=1}^T\left(c_t(\mathbf{x}_t)-c_t(\mathbf{x}_t^{c})\right)+\sum_{t=1}^T\left(c_t(\mathbf{x}_t^{c})-c_t(\textbf{x}_*)\right)}{\eta^c}\\
&+\eta^cG^2D^2T\\
\leq&\left(\ln 3+ \frac{3}{4}\right)2GD\sqrt{T}+\frac{1}{2}GD\sqrt{T}\\
=&2(1+\ln 3)GD\sqrt{T}
\end{split}
\end{equation*}
where $\x_*=\min_{\x\in\mathcal{D}}\sum_{t=1}^Tf_t(\x)$, and the last inequality follows from \eqref{lemma1e3} and \eqref{lemma2e3}.

Next, to achieve the regret of \eqref{t1sc}, we upper bound $R(T)$ by making use of the properties of  $s_t^{\ell}$. For every grid point $\eta$, we have
\begin{equation}
\begin{split}
\label{rit}
&R(T)\\
\overset{\eqref{defn:regret}}{=}&\sum_{t=1}^Tf_t(\mathbf{x}_t)-\sum_{t=1}^Tf_t(\mathbf{x}_*)\\
\overset{\eqref{defn:convex}}{\leq}& \sum_{t=1}^T\mathbf{g}_t^{\top}(\mathbf{x}_t-\mathbf{x}_*)\\
\overset{\eqref{defn:surrogates}}{=}&\frac{\sum_{t=1}^T-s^{\eta}_t(\textbf{x}_*)+\eta^2G^2\|\mathbf{x}_*-\mathbf{x}_t\|^2}{\eta}\\
=&\frac{\sum_{t=1}^T\left(s_t^{\eta}(\mathbf{x}_t)-s_t^{\eta}(\mathbf{x}_t^{\eta,s})\right)+\sum_{t=1}^T\left(s_t^{\eta}(\mathbf{x}_t^{\eta,s})-s_t^{\eta}(\textbf{x}_*)\right)}{\eta}\\
&+\sum_{t=1}^T\eta G^2\|\mathbf{x}_t-\mathbf{x}_*\|^2\\
\leq &\frac{2\ln\left(\sqrt{3}\left(\frac{1}{2}\log_2 T+3\right)\right)+1+\log T}{\eta}\\
&+\sum_{t=1}^T\eta G^2\|\mathbf{x}_t-\mathbf{x}_*\|^2\\
=&\eta V_T^s+\frac{2\ln\left(\sqrt{3}\left(\frac{1}{2}\log_2 T+3\right)\right)+1+\log T}{\eta}
\end{split}
\end{equation}
where $V_T^s=\sum_{t=1}^TG^2\|\x_t-\x_*\|^2$, and the inequality comes from \eqref{lemma1e1} and \eqref{lemma2e1}. Define
\begin{equation*}
\begin{split}
A=2\ln\left(\sqrt{3}\left(\frac{1}{2}\log_2 T+3\right)\right)+1+\log T\geq 1.
\end{split}
\end{equation*}
The optimal $\hat{\eta}$ to minimize the right hand side of \eqref{rit} is
\begin{equation}
\begin{split}
\label{opteta}
\hat{\eta}=&\sqrt{\frac{A}{V_T^s}}\geq \frac{1}{5GD\sqrt{T}}.
\end{split}
\end{equation}
If $\hat{\eta}\leq\frac{1}{5GD}$, then by construction their exists a grid point $\eta\in[\frac{\hat{\eta}}{2},\hat{\eta}]$, and thus
\begin{equation*}
\begin{split}
&R(T)\leq {\eta}V_T^s+\frac{A}{{\eta}}\leq \hat{\eta}V_T^s+\frac{2A}{\hat{\eta}}=3\sqrt{V_T^sA}.
\end{split}
\end{equation*}
On the other hand, if $\hat{\eta}>\frac{1}{5GD}$, then by \eqref{opteta} we get
\begin{equation*}
V_T^s\leq 25G^2D^2A.
\end{equation*}
Thus for $\eta_1=\frac{1}{5GD}$, we have
\begin{equation*}
R(T)\leq 10GDA.
\end{equation*}
Overall, we obtain
\begin{equation*}
\begin{split}
R(T)\leq& 3\sqrt{V_T^sA}+10GDA.
\end{split}
\end{equation*}

Finally, we upper bound the regret by using  the properties exp-concave surrogate loss functions. For every grid point  $\eta$, we have
\begin{equation*}
\begin{split}
&R(T)\\
\overset{\eqref{defn:regret}}{=}&\sum_{t=1}^Tf_t(\mathbf{x}_t)-\sum_{t=1}^Tf_t(\mathbf{x}_*)\\
\overset{\eqref{defn:convex}}{\leq}& \sum_{t=1}^T\mathbf{g}_t^{\top}(\mathbf{x}_t-\mathbf{x}_*)\\
\overset{\eqref{elll}}{=}&\frac{\sum_{t=1}^T-\ell^{\eta}_t(\textbf{x}_*)+\eta^2\left(\mathbf{g}_t^{\top}(\mathbf{x}_t-\mathbf{x}_*)\right)^2}{\eta}\\
=&\frac{\sum_{t=1}^T\left(\ell_t^{\eta}(\mathbf{x}_t)-\ell_t^{\eta}(\mathbf{x}_t^{\eta,\ell})\right)}{\eta}+\eta\sum_{t=1}^T\left(\mathbf{g}_t^{\top}(\mathbf{x}-\mathbf{x}_*)\right)^2\\
&+\frac{\sum_{t=1}^T\left(\ell_t^{\eta}(\mathbf{x}_t^{\eta,\ell})-\ell_t^{\eta}(\textbf{x}_*)\right)}{\eta}\\
\leq &\frac{2\ln\left(\sqrt{3}\left(\frac{1}{2}\log_2 T+3\right)\right)+10d\log T}{\eta}\\
&+\eta\sum_{t=1}^T\left(\mathbf{g}_t^{\top}(\mathbf{x}_*-\mathbf{x}_t)\right)^2\\
=&\eta V_T^{\ell}+\frac{2\ln\left(\sqrt{3}\left(\frac{1}{2}\log_2 T+3\right)\right)+10d\log T}{\eta}
\end{split}
\end{equation*}
Where $V_T^{\ell}=\sum_{t=1}^T\left(\left(\x_t-\x_*\right)^{\top}\mathbf{g}_t\right)^2$, and the last inequality comes from \eqref{lemma1e2} and \eqref{lemma2e2}. Define $$B=2\ln\left(\sqrt{3}\left(\frac{1}{2}\log_2 T+3\right)\right)+10d\log T.$$
By similar arguments, we get
\begin{equation*}
\begin{split}
R(T)\leq 3\sqrt{V_T^{\ell}B}+10GDB.
\end{split}
\end{equation*}
\end{proof}
\subsection{Proof of Corollary \ref{cor}}
\label{pcor}
\begin{proof}
For $\alpha$-exp-concave functions, we have
\begin{equation*}
\begin{split}
&R(T)\\
\leq& \sum_{t=1}^T\mathbf{g}_t^{\top}(\mathbf{x}_t-\mathbf{x}_*)-\frac{\beta}{2}V_T^{\ell}\\
\leq &3\sqrt{V_T^{\ell}\left(2\ln\left(\sqrt{3}\left(\frac{1}{2}\log_2 T+3\right)\right)+10d\log T\right)}\\
&+10GD\left(2\ln\left(\sqrt{3}\left(\frac{1}{2}\log_2 T+3\right)\right)+10d\log T\right)\\
&-\frac{\beta}{2}V_T^{\ell}\\
\leq & \frac{3\gamma}{2}V_T^{\ell}+\left(10GD+\frac{3}{2\gamma}\right)\left(2\ln\left(\sqrt{3}\left(\frac{1}{2}\log_2 T+3\right)\right)\right.\\
&\left.+10d\log T\right)-\frac{\beta}{2}V_T^{\ell}
\end{split}
\end{equation*}
where the last inequality is based on $\sqrt{xy}\leq \frac{\gamma}{2}x+\frac{y}{2\gamma}$ for all $x,y,\gamma>0$, The result follows from $\gamma=\frac{\beta}{3}$.\\
For $\lambda$-strongly convex functions, we have
\begin{equation*}
\begin{split}
&R(T)\\
\leq& \sum_{t=1}^T\mathbf{g}_t^{\top}(\mathbf{x}_t-\mathbf{x}_*)-\frac{\lambda}{2}\|\x_t-\x_*\|^2\\
\leq&3\sqrt{V_T^s\left(2\ln\left(\sqrt{3}\left(\frac{1}{2}\log_2 T+3\right)\right)+1+\log T\right)}\\
&+10GD\left(2\ln\left(\sqrt{3}\left(\frac{1}{2}\log_2 T+3\right)\right)+1+\log T\right)\\
&-\frac{\lambda}{2G^2}V^s_T\\
\leq&\frac{3\gamma V_T^s}{2}+\left(10GD+\frac{3}{2\gamma}\right)\left(2\ln\left(\sqrt{3}\left(\frac{1}{2}\log_2 T+3\right)\right)\right.\\
&\left.+1+\log T\right)-\frac{\lambda}{2}V_T^s
\end{split}
\end{equation*}
where the last inequality is due to $\sqrt{xy}\leq \frac{\gamma}{2}x+\frac{y}{2\gamma}$ for all $x,y,\gamma>0$, and the result follows from $\gamma=\frac{\lambda}{3G^2}$.
\end{proof}
\begin{figure*}[t]
        \begin{subfigure}[b]{0.48\textwidth}
        \includegraphics[width=\textwidth]{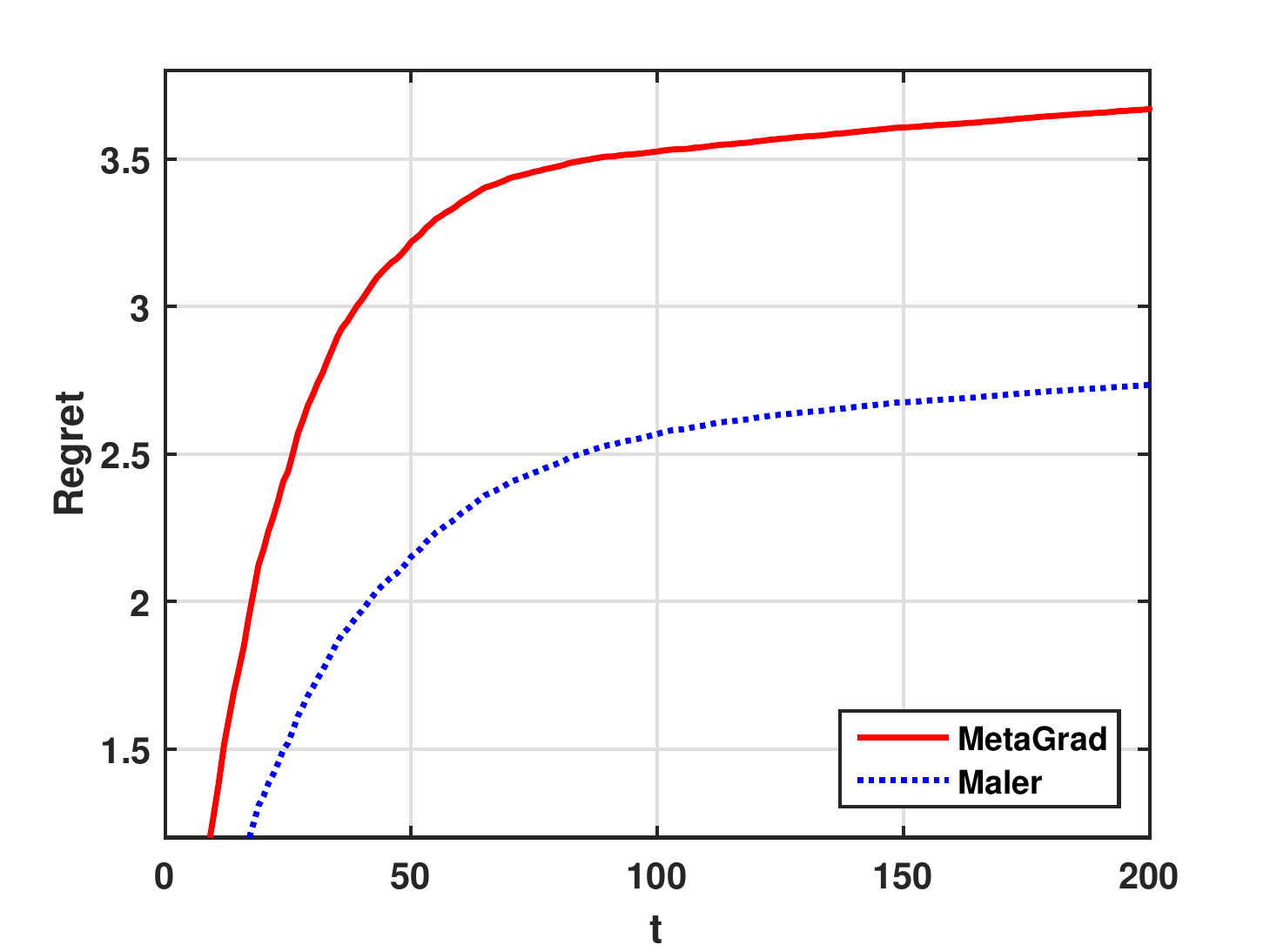}
        \caption{Online regression}
        \label{fig:MNISTCNNA}
    \end{subfigure}
    \hspace{-.2in}
    ~ 
    \begin{subfigure}[b]{0.48\textwidth}
        \includegraphics[width=\textwidth]{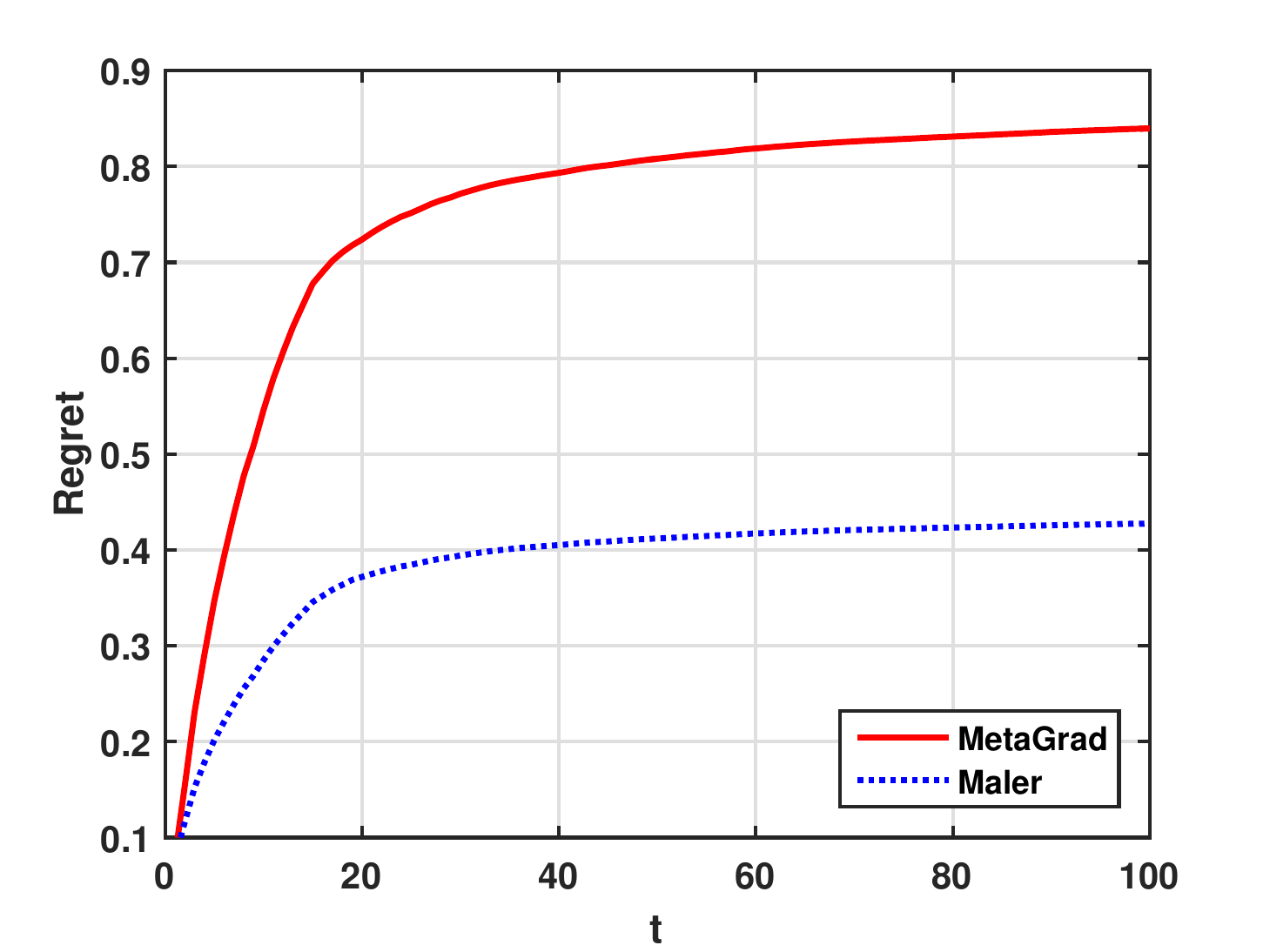}
        \caption{Online classification}
        \label{fig:CIFAR10CNNA}
    \end{subfigure}

    \caption{Emprecial results of Maler and MetaGrad for online regression and classification}
    ~ 
\end{figure*}
\section{Experiments}
In this section, we present empirical results on different online learning tasks to evaluate the proposed algorithm. We choose MetaGrad as the baseline method.
\subsection{Online Regression}
We consider  mini-batch least mean square regression with $\ell_2$ regularizer, which is a classic problem belonging to online strongly convex optimization. In each round $t$, a small batch of training examples $\{(\x_{t,1},y_{t,1}),\dots,(\x_{t,n},y_{t,n})\}$ arrives, and at the same time, the learner makes a prediction of the unknown parameter $\mathbf{w}_*$, denoted as $\mathbf{w}_t$, and suffers a loss, defined as
\begin{equation}
\begin{split}
f_t(\mathbf{w})=\frac{1}{n}\sum_{i=1}^n\left(\mathbf{w}^{\top}\x_{t,i}-y_{t,i}\right)^2+\lambda \|\mathbf{w}\|^2.
\end{split}
\end{equation}
We conduct the experiment on a symmetric data set, which is constructed as follows. We sample $\mathbf{w}_*$ and feature vectors $\x_{t,i}$ uniformly at random from the $d$-ball of diameter 1 and 10 respectively, and generate $y_{t,i}$ according to a linear model: $y_{t,i}=\mathbf{w}_*^{\top}\x_{t,i}+\eta_t$, where the noise is drawn from a normal distribution. We set batch size $n=200$, $\lambda=0.001$, $d=50$, and $T=200$. The regret v.s. time horizon is shown in Figure 1(a). It can be seen that Maler achieves faster convergence rate than MetaGrad.
\subsection{Online Classification}
Next, we consider online classification by using logistic regression. In each round $t$, we receive a batch of training examples $\{(\x_{t,1},y_{t,1}),\dots,(\x_{t,n},y_{t,n})\}$, and choose a linear classifier $\mathbf{w}$. After that, we suffer a logistic loss
\begin{equation}
f_t(\mathbf{w})=\frac{1}{n}\sum_{i=1}^n\log(1+\exp(-y_{t,i}\mathbf{w}_{t}^{\top}\x_{t,i}))
\end{equation}
which is exp-concave. We conduct the experiments on a classic real-world data set a9a \citep{CC01a}. We scale all feature vectors to the unit ball, and restrict the decision  set $\mathcal{D}$ to be a ball of radius 0.5 and centered at the origin, so that Assumptions 1 and 2 are satisfied. We set batch size  $n=200$, and $T=100$. The regret v.s. time horizon is shown in Figure 1(b). It can be seen that Maler performs better than MetaGrad. Although the worst-case regret bounds of Maler and MetaGrad for exp-concave loss are on the same order, the experimental results are not surprising since Maler enjoys a tighter data-dependant regret bound than that of MetaGrad.
\section{Conclusion and Future Work}
In this paper, we propose a universal algorithm for online convex optimization, which achieves the optimal $O(\sqrt{T})$, $O(d\log T)$ and $O(\log T)$ regret bounds for general convex, exp-concave and strongly convex functions respectively, and enjoys a new type of data-dependent bound. The main idea is to consider different types of learning algorithms and learning rates at the same time. Experiments on online regression and online classification problems demonstrate the effectiveness of our method. In the future, we will investigate whether our proposed algorithm can extend to achieve border adaptivity in various directions, for example, adapting to changing environments \citep{hazan2007adaptive} and adapting to data structures \citep{reddi2018convergence}.
\newpage
\balance
\bibliography{sample2e-2019}
\bibliographystyle{apalike}
\newpage
\onecolumn
\appendix
\section{Proof of Lemma 2}
\begin{proof}
The proof technique is standard, and can be found in \cite{zinkevich2003online,hazan2016introduction}.

First, we prove the regret bound of \eqref{lemma2e1}. Note that by Definition \ref{defn:stconvex},  $s_t^{\eta}(\x)$ is $2\eta^2G^2$-strongly convex. For convince, we denote $\alpha_{t+1}=1/(2\eta^2G^2t)$, $\lambda^s=2\eta^2G^2$, and define the upper bound of the gradients of $s_t^{\eta}(\x)$ as
$$\max_{\x\in\mathcal{D}}\|\nabla s_t^{\eta}(\x)\|=\max_{\x\in\mathcal{D}} \|\eta\mathbf{g}_t+2\eta^2G^2(\mathbf{x}-\mathbf{x}_t)\|\leq G\eta +2\eta^2 G^2D=:G^s.$$
By the update rule of $\x^{\eta,s}_{t+1}$, we have
\begin{equation}
\begin{split}
\|\x^{\eta,s}_{t+1}-\u\|=&\left\|\Pi_{\mathcal{D}}^{I_d}\left(\x^{\eta,s}_{t}
-\alpha_{t+1}\nabla s_t^{\eta}(\x^{\eta,s}_{t})\right)-\u\right\|\\
\leq&\left\|\x^{\eta,s}_{t}
-\alpha_{t+1}\nabla s_t^{\eta}(\x^{\eta,s}_{t})-\u\right\|\\
=&\|\x^{\eta,s}_{t}-\u\|^2+\alpha_{t+1}^2\|\nabla s_t^{\eta}(\x^{\eta,s}_{t}) \|^2-2\alpha_{t+1}(\x^{\eta,s}_{t}-\u)^{\top}\nabla s_t^{\eta}(\x^{\eta,s}_{t}).
\end{split}
\end{equation}
Hence,
\begin{equation}
\label{sc111}
2(\x^{\eta,s}_{t}-\u)^{\top}\nabla s_t^{\eta}(\x^{\eta,s}_{t})\leq\frac{\|\x^{\eta,s}_{t}-\u\|-
\|\x^{\eta,s}_{t+1}-\u\|^2}{\alpha_{t+1}}+\alpha_{t+1}(G^s)^2.
\end{equation}
Summing over $1$ to $T$ and  applying definition 2, we get
\begin{equation}
\begin{split}
2\sum_{t=1}^Ts_t^{\eta}(\mathbf{x}_t^{\eta,s})-2\sum_{t=1}^Ts_t^{\eta}(\u)\leq& \sum_{t=1}^T\|\x^{\eta,s}_{t}-\u\|^2\left(\frac{1}{\alpha_{t+1}}-\frac{1}{\alpha_{t}}-\lambda^s\right)+\left(G^s\right)^2\sum_{t=1}^T\alpha_{t+1}\\
\leq&\frac{\left(G^s\right)^2}{\lambda^s}(1+\log T).
\end{split}
\end{equation}
Note that $\eta\leq \frac{1}{5DG}$. We have
\begin{equation}
\begin{split}
\left(G^s\right)^2&=G^2\eta^2+4\eta^3G^3D+4\eta^4G^4D^2\leq G^2\eta^2+\frac{4\eta^2G^2}{5}+\frac{4\eta^2G^2}{25}\leq 2\eta^2G^2=\lambda^s.
\end{split}
\end{equation}

Next, we prove the regret bound of \eqref{lemma2e2}. We start with the following inequality
\begin{equation}
\begin{split}
\nabla \ell^{\eta}_t(\mathbf{x})(\nabla \ell^{\eta}_t(\mathbf{x}))^{\top}=&\eta^2 \mathbf{g}_t\mathbf{g}_t^{\top}+4\eta^3\mathbf{g}_t(\mathbf{x}-\mathbf{x}_t)^{\top}\mathbf{g}_t\mathbf{g}_t^{\top}+4\eta^4\mathbf{g}_t\mathbf{g}_t^{\top}(\mathbf{x}-\mathbf{x}_t)(\mathbf{x}-\mathbf{x}_t)^{\top}\mathbf{g}_t\mathbf{g}_t^{\top}\\
=&\eta^2 \mathbf{g}_t\mathbf{g}_t^{\top}+\mathbf{g}_t\left(4\eta^3(\mathbf{x}-\mathbf{x}_t)^{\top}\mathbf{g}_t+4\eta^4\left((\mathbf{x}-\mathbf{x}_t)^{\top}\mathbf{g}_t\right)^2\right)\mathbf{g}_t^{\top}\\
\preceq& 2\eta^2 \mathbf{g}_t\mathbf{g}_t^{\top}=\nabla^2 \ell_t^{\eta}(\mathbf{x})
\end{split}
\end{equation}
where $\nabla^2 \ell_t^{\eta}(\mathbf{x})$ denotes the Hessian matrix. The inequality implies that $\nabla^2 \ell_t^{\eta}(\mathbf{x})\succeq\nabla \ell^{\eta}_t(\mathbf{x})(\nabla \ell^{\eta}_t(\mathbf{x}))^{\top}$. According to Lemma 4.1 in \cite{hazan2016introduction}, $\ell_t^{\eta}(\x)$ is 1-exp-concave. Next, we prove that the gradient of $\ell_t^{\eta}(\x)$ can be upper bounded as follows
\begin{equation}
\max_{\x\in\mathcal{D}}\|\nabla \ell_t^{\eta}(\mathbf{x})\|\leq \eta G + 2\eta^2 G^2D\leq \frac{7}{25D}=G^{\ell}.
\end{equation}
By Theorem 4.3 in  \cite{hazan2016introduction}, we have
\begin{equation}
\sum_{t=1}^T\ell_t^{\eta}(\mathbf{x}_t^{\eta,\ell})-\sum_{t=1}^T\ell_t^{\eta}(\u)\leq 5(1+G^{\ell}D)d\log T\leq10d\log T.
\end{equation}

Finally, we prove the regret bound of \eqref{lemma2e3}. Note that the gradient of $c_t(\x)$ is upper bounded by $\max_{\x\in\mathcal{D}}\|\nabla c_t(\x)\|\leq\eta^cG$. Define $m_t=\frac{D}{\eta^cG\sqrt{t}}$. By the convexity of $c_t(\mathbf{x})$, we have $\forall \u\in\mathcal{D}$,
\begin{equation}
\label{substo}
c_t\left(\mathbf{x}_t^c\right)-c_t\left(\u\right)\leq \left(\mathbf{x}_t^c-\u\right)^{\top}\nabla c_t\left(\mathbf{x}_t^c\right).
\end{equation}
On the other hand, according to the update rule of $\mathbf{x}_{t+1}^c$, we have
\begin{equation}
\begin{split}
\|\mathbf{x}_{t+1}^c-\u\|^2=&\|\Pi^{I_d}_{\mathcal{D}}\left(\mathbf{x}^c_t-m_t\nabla c_t(\mathbf{x}^c_t)\right)-\u\|^2\\
\leq&\|\x_t^c-m_t\nabla c_t\left(\x_t^c\right)-\u\|^2\\
=&\|\x_t^c-\u\|^2+m_t^2\|\nabla c_t\left(\x_t^c\right)\|^2-2 m_t \left(\x_t^c-\u\right)^{\top}\nabla c_t\left(\x_t^c\right)
\end{split}
\end{equation}
where the inequality follows from Theorem 2.1 in \cite{hazan2016introduction}. Hence,
\begin{equation}
\begin{split}
&2\left(\x_t^c-\u\right)^{\top}\nabla c_t\left(\x_t^c\right)\\
\leq&\frac{\|\x_t^c-\u\|^2-\|\x_{t+1}^c-\u\|^2}{m_t}+m_t\|\nabla c_t\left(\x_t^c\right)\|^2\\
\leq& \frac{\|\x_t^c-\u\|^2-\|\x_{t+1}^c-\u\|^2}{m_t}+m_t(\eta^cG)^2
\end{split}
\end{equation}
Substituting the above inequality into \eqref{substo} and summing over $T$, we have
\begin{equation}
\begin{split}
\sum_{t=1}^T c_t(\x_t^c)-c_t(\u)\overset{\eqref{defn:convex}}{\leq}& \sum_{t=1}^T\left(\x_t^c-\u\right)^{\top}\nabla c_t\left(\x_t^c\right)\\
\leq&\frac{1}{2}\sum_{t=1}^T\|\x^c_t-\u\|^2\left(\frac{1}{m_t}-\frac{1}{m_{t-1}}\right)+\frac{(\eta^cG)^2}{2}\sum_{t=1}^Tm_t\\
\leq&D^2\frac{1}{2m_T}+\frac{(\eta^cG)^2}{2}\sum_{t=1}^Tm_t\\
\leq& \frac{3}{2}\eta^cGD\sqrt{T}\leq \frac{3}{4}
\end{split}
\end{equation}
where the last inequality is due to $\eta^c=\frac{1}{2GD\sqrt{T}}$.
\end{proof}
\end{document}